\documentclass[letterpaper, 10 pt, conference]{ieeeconf}  

\input{preamble}

\IEEEoverridecommandlockouts                              

\title{\LARGE \bf
GPU-Accelerated Continuous-Time Successive Convexification for Contact-Implicit Legged Locomotion
}

\author{Samuel C. Buckner$^{1}$ and Purnanand Elango$^{2*}$
\thanks{$^{1}$Department of Aeronautics and Astronautics, University of Washington, Seattle, WA 98195.}%
\thanks{S. C. Buckner interned at MERL during the development of this work.}%
\thanks{$^{2}$Mitsubishi Electric Research Laboratories, Cambridge, MA 02139.}%
\thanks{$^{*}$Corresponding author. Email: \texttt{elango@merl.com}}}

\begin{document}
\raggedbottom 

\maketitle
\thispagestyle{empty}
\pagestyle{empty}
%
\begin{abstract}
Contact-implicit trajectory optimization (CITO) enables the automatic discovery of contact sequences, but most methods rely on fine time discretization to capture all contact events accurately, which increases problem size and runtime while tying solution quality to grid resolution. We extend the recently proposed sequential convex programming (SCP) approach for trajectory optimization, continuous-time successive convexification (\ctscvx), to CITO by introducing integral cross-complementarity constraints, which eliminate the risk of missing contact events between discretization nodes while preserving the flexibility of contact mode changes. The resulting framework, contact-implicit successive convexification ({\ciscvx}), models full multibody dynamics in maximal coordinates, including stick-slip friction and partially elastic impacts. To handle complementarity constraints, we embed a backtracking homotopy scheme within SCP for reliable convergence. We implement this framework in a stand-alone Python software, leveraging JAX for GPU acceleration and a custom canonical-form parser for the convex subproblems of SCP to avoid the overhead of general-purpose modeling tools such as CVXPY. We demonstrate {\ciscvx} on diverse legged-locomotion tasks. In particular, we validate the approach in MuJoCo with the Gymnasium HalfCheetah model against the MuJoCo MPC baseline, showing that a tracking simulation with the optimized torque profiles from {\ciscvx} produces physically consistent trajectories with lesser energy consumption. We also show that the resulting software achieves faster solve times than existing state-of-the-art SCP implementations by over an order of magnitude, thereby demonstrating a practically important contribution to scalable real-time trajectory optimization.
\end{abstract}
\section{Introduction}
Trajectory optimization is a central tool for generating agile behaviors in legged robots \cite{wensing2023optimization}. Contact-implicit trajectory optimization (CITO) methods are particularly attractive because they eliminate the need to predefine contact schedules. However, most existing CITO formulations rely on fine time discretization to capture contact events. This increases the number of decision variables and constraints, coupling solution quality tightly to the discretization resolution and creating an unfavorable accuracy-runtime tradeoff \cite{patel2019contact}.

To address this limitation, we build on the sequential convex programming (SCP) approach, which iteratively approximates and solves nonconvex problems. While the usage of SCP has been addressed for CITO \cite{onol2019contact,onol2020tuning}, we instead consider the recently proposed continuous-time successive convexification ({\ctscvx}) framework \cite{elango2025ctscvx}, which notably enforces constraints in continuous time and thus permits high-fidelity, feasible solutions over coarse discretization grids. Our key idea is to extend continuous-time satisfaction to contact constraints, specifically cross-complementarity constraints \cite{baumrucker2009mpec}. We introduce an integral form of cross-complementarity constraints expressed in continuous time, which ensures that contact mode changes occur at discretization nodes while eliminating the risk of missing contact events between nodes. In conjunction with time-dilation \cite{elango2025ctscvx}, where nodes are allowed to stretch or shrink over time, the generality of the resulting contact mode sequence is ensured.

We formulate the full multibody dynamics of a robot in maximal coordinates rather than using reduced-order approximations such as SRBD \cite{wensing2023optimization} or centroidal dynamics \cite{orin2013centroidal}. Although maximal coordinates enlarge the state dimension, the associated dynamics and constraints are highly sparse, enabling efficient GPU parallelization. The model supports stick–slip friction, consistent with the maximum dissipation principle, and partially elastic impacts \cite{stewart2011dynamics}. To handle complementarity constraints, we employ a homotopy scheme with backtracking that tightens the relaxation during the SCP iterations, ensuring reliable convergence. We implement the approach in stand-alone Python software that leverages JAX for GPU-accelerated differentiation and linearization, together with a custom parser for the convex subproblems in SCP that avoids the overhead of general-purpose modeling tools such as CVXPY. We demonstrate the approach on two-dimensional legged robot models, including the Gymnasium HalfCheetah model \cite{towers2024gymnasium}.
\begin{figure*}[t]
    \centering
    \includegraphics[width=0.95\linewidth]{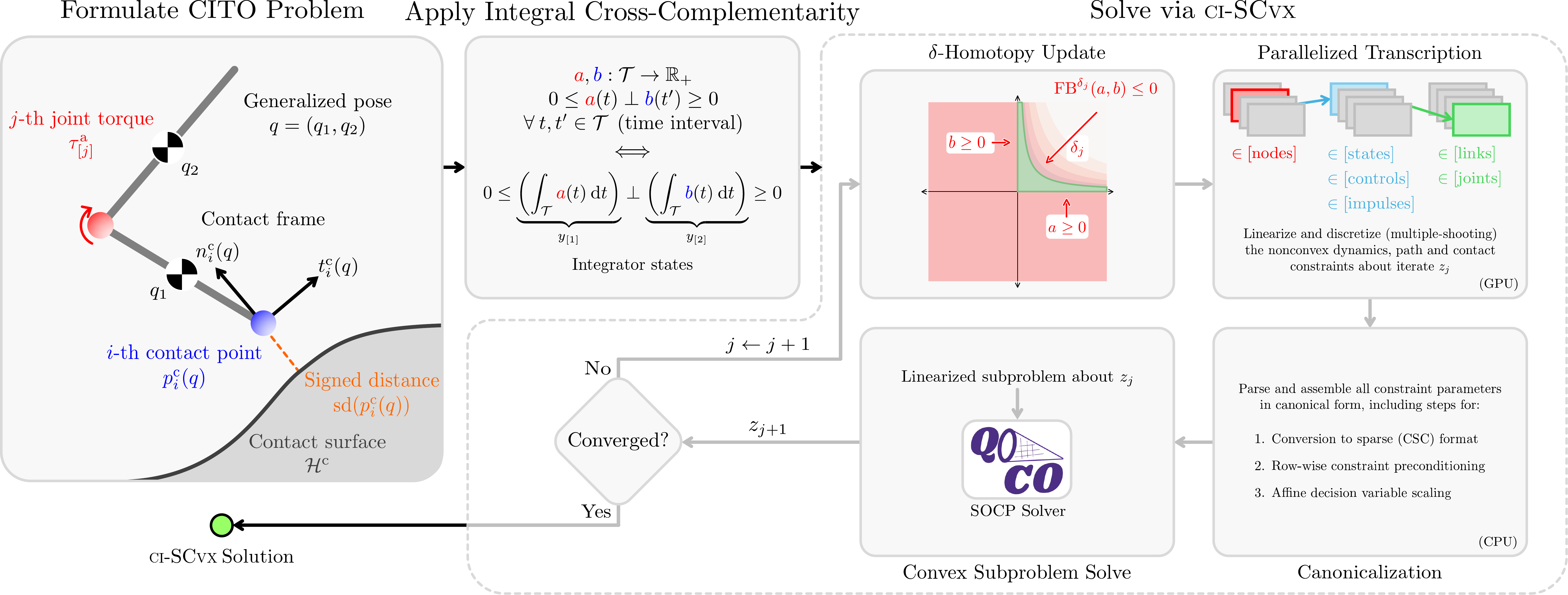}
    \caption{Overview of the proposed framework, including problem formulation, integral cross-complementarity and the \ciscvx\:algorithm.}
    \label{fig:architecture}
\end{figure*}
\subsection{Related Work}
CITO has been studied extensively to generate locomotion solutions without prescribing contact schedules. Early approaches used direct collocation with complementarity constraints to enforce non-penetration and frictional contact \cite{direct2013posa}, while contact-invariant optimization \cite{mordatch2012discovery} handled contact through smooth penalties that encourage sticking and end-effector alignment, with dynamics enforced softly. These methods typically require fine time discretization to capture contact events accurately, yielding large-scale problems whose solution quality depends strongly on the grid resolution. To mitigate the risk of missing contact events between nodes, cross-complementarity constraints were proposed \cite{baumrucker2009mpec}. More recent formulations such as FESD-J \cite{nurkanovic2024fesd-j} and moving finite elements \cite{kazi2025optimal} extend this idea by enabling exact switch detection for systems with impacts and Coulomb friction, ensuring that contact events occur at finite-element boundaries. However, these are still inherently discrete-time constructions.

Continuous-time trajectory optimization frameworks like {\ctscvx} \cite{elango2025ctscvx} enforce constraints in continuous time and thereby achieve accurate solutions with coarse discretization, but they have not been extended to CITO. In terms of modeling, many works adopt reduced-order approximations such as centroidal dynamics \cite{orin2013centroidal} or single-rigid-body dynamics (SRBD) \cite{wensing2023optimization}. While these simplifications can reduce the size of optimization problems and sometimes eliminate the risk of missing contact events \cite{winkler2018gait}, they often neglect whole-body effects and can lead to conservative solutions that underutilize the system's full range of motion.

\subsection{Contributions}
This work introduces a continuous-time CITO framework for locomotion tasks and makes the following contributions:
\begin{enumerate}
\item The contact-implicit successive convexification ({\ciscvx}) framework, an SCP-based approach for CITO that ensures continuous-time constraint satisfaction.
\item Integral cross-complementarity constraints that ensure contact mode changes occur at discretization nodes. Combined with time-dilation \cite{elango2025ctscvx}, this preserves generality in contact mode sequences.
\item A backtracking homotopy scheme that progressively tightens the complementarity constraint relaxation within SCP iterations, ensuring reliable convergence.
\item A stand-alone Python software that leverages JAX for GPU-accelerated differentiation and linearization, together with a custom parser for the convex subproblems in SCP that avoids the overhead of general-purpose modeling tools such as CVXPY.
\item Demonstration on two-dimensional locomotion tasks, including the Gymnasium HalfCheetah model \cite{towers2024gymnasium}, with comparisons to a state-of-the-art baseline.
\end{enumerate}
%
\section{Dynamical Model with Contact and Impact}\label{sec:system-model}
We adopt the following notation for the remainder of the development. Concatenation of vectors $u\in\bR^n$ and $v\in\bR^m$ is $(u,v)\in\bR^{n+m}$. The scalar components of $v\in\bR^n$ are denoted by $v_{[i]}$, for $i=1,\ldots,n$. The ramp (ReLU) function is denoted by $\absplus{x}$. Smooth pseudo-Huber approximations of the Euclidean norm, absolute value, and ramp functions with unit steepness parameter are denoted by $\snorm{x}$, $\sabs{x}$, and $\sabsplus{x}$. For vector arguments, $\sabsplus{x}$ is understood componentwise.

We use maximal coordinates to model the dynamics of a multi-link robot. Let $q\in\bR^{n^q}$ denote the generalized position (pose) and $v\in\bR^{n^q}$ the generalized velocity. The equations of motion are obtained by applying Newton’s second law to each link, accounting for forces and torques from gravity, joint reaction forces, joint torques, joint stiffness and damping, and contact forces. The robot has $\ncontacts$ contact points with positions, expressed in the inertial (world) frame, given by $\posc{i}(q)\in\bR^{\ndim}$ for $i=1,\ldots,\ncontacts$. The Jacobian of the $i$th contact point (the contact Jacobian) is $\Jc{i}(q)=\nabla \posc{i}(q)$. 

Contact occurs on a smooth surface
$\surfc = \{p\in\bR^{\ndim}\,|\,\surfcfun(p)=0\}$,
where $\surfcfun$ is a continuously differentiable function and $n^\mathrm{d}\in\{2,3\}$ is the dimension of the position space. At each contact point, we define a contact frame consisting of the unit surface normal
$\normalc{i}(q)=\nabla \signdist(\posc{i}(q))^\top \in \bR^{\ndim}$ and $(\ndim-1)$ unit tangent vectors given by the columns of $\tangentc{i}(q)\in\bR^{\ndim\times(\ndim-1)}$. Together, the columns of $\tangentc{i}(q)$ and $\normalc{i}(q)$ form a right-handed basis. The rotation matrix that maps vectors from the contact frame to the world frame is
$\rotmatc{i}(q) = [\,\tangentc{i}(q)\,\ \normalc{i}(q)\,]$. Fig. \ref{fig:architecture} shows the frame definitions. The contact force at the $i$th contact point, represented in the contact frame, is
$\phic{i}=(\phict{i},\phicn{i})$, where
$\phicn{i}\in\bR$ is the normal component along $\normalc{i}(q)$ and
$\phict{i}\in\bR^{\ndim-1}$ is the tangential component.
Similarly, the contact impulse due to impact is
$\Phic{i}=(\Phict{i},\Phicn{i})$ with the same component-wise interpretation. 

Let $\jointcnstr(q)=\zeros{\njoints}$ impose the kinematic constraints that connect links at the joints and restrict degrees of freedom according to joint type. The effect of joint reaction forces on the generalized acceleration $\dot{v}$ is captured by $\nabla\jointcnstr(q)^\top\jointforce$, where $\jointforce$ concatenates all joint reaction forces in the world frame. An impact at any contact point induces a joint reaction impulse that instantaneously changes the generalized velocity via $\nabla\jointcnstr(q)^\top\jointimpulse$, where $\jointimpulse$ concatenates joint reaction impulses in the world frame.
%
\subsection{Continuous-Time Model}
The equations of motion are given by
\begin{subequations}
\begin{align}
    \dot{q} ={} & v,\\
    M(q)\dot{v} ={} & W(q,v,\jointtorque{},\gravity) + \Jc{}(q)^\top\rotmatc{i}(q)\phic{} \label{eq:ct-dynamics:acc}\\
    & + \nabla\jointcnstr(q)^\top\jointforce,\notag
\end{align}\label{eq:ct-dynamics}%
\end{subequations}
where $\Jc{}(q)^\top$ $=$ $[\Jc{1}(q)^\top$ $\cdots$ $\Jc{\ncontacts}(q)^\top]$, $\rotmatc{}(q)$ $=$ $\mr{blockdiag}(\rotmatc{1}(q),\ldots,\rotmatc{\ncontacts}(q))$, and $\phic{}$ $=$ $(\phic{1},\ldots,\phic{\ncontacts})$. $M(q)$ is the mass matrix, and $W(q,v,\jointtorque{},\gravity)$ collects the generalized forces and torques due to gravity (with $\gravity$ the gravitational acceleration), joint torques $\jointtorque{}$, and joint stiffness and damping. To model physically realistic contact, we require (i) a no-penetration condition expressed as a complementarity constraint between the normal contact force and signed distance to the contact surface, and (ii) a tangential contact force determined by the maximum dissipation principle subject to a friction cone constraint
\begin{subequations}
\begin{align}
    & 0\le \phicn{} \perp \signdist(\posc{i}(q)) \ge 0,\label{eq:no-penetration}\\
    & \phict{} \in \underset{\phi^{\mr{t}}}{\mr{argmin}}~\phi^{\mr{t}}{}^\top \tangentcvel{i}(q,v)~~\mr{s.t.}~\norm{\phi^{\mr{t}}}\le\mu_i\phicn{i},\label{eq:max-diss}
\end{align}\label{eq:no-penetration-max-diss}%
\end{subequations}
for $i=1,\ldots,\ncontacts$, where $\mu_i$ is the friction coefficient and $\tangentcvel{i}(q,v)=\tangentc{i}(q)^\top\Jc{i}(q)v$ is the tangential velocity at the $i$th contact. The first-order optimality conditions for \eqref{eq:max-diss} are%
\begin{subequations}
\begin{align}
& \zeros{\ndim-1} = \tangentcvel{i}(q,v)+2\gamma_i\phict{i},\label{eq:max-diss-tangent-contact-force-KKT:stationarity}\\
& 0 \le \gamma_i \perp (\mu_i\phicn{i})^2-|\phict{i}|^2 \ge 0,
\end{align}\label{eq:max-diss-tangent-contact-force-KKT}%
\end{subequations}
where $\gamma_i$ is the Lagrange multiplier and we denote $\gamma=(\gamma_1,\ldots,\gamma_{\ncontacts})$. The friction cone constraint is written in squared form to ensure differentiability at the boundary. When the $i$th contact is active (i.e., signed distance is zero), $\gamma_i$ is a continuous function of time. Two issues arise with \eqref{eq:max-diss-tangent-contact-force-KKT}: (i) the $\norm{\cdot}^2$ term causes numerical ill-conditioning, and (ii) when the $i$th contact is inactive (i.e., signed distance is positive), $\phic{i}=\zeros{\ndim}$, so \eqref{eq:max-diss-tangent-contact-force-KKT:stationarity} reduces to the non-physical case that the contact point's tangential velocity is zero, i.e., $\tangentc{i}(q)^\top\Jc{i}(q)v=\zeros{\ndim-1}$. We address both by rewriting \eqref{eq:max-diss-tangent-contact-force-KKT} as%
\begin{subequations}
\begin{align}
    & \zeros{\ndim-1} = \phicn{i}\big(\tangentcvel{i}(q,v) + \gamma_i\nabla\snorm{\phict{i}}^\top\big), \label{eq:max-diss-tangent-contact-force-KKT-reformulate:stationarity}\\
    & 0 \le \gamma_i \perp \sabs{\mu_i\phicn{i}}-\snorm{\phict{i}} \ge 0,
\end{align}\label{eq:max-diss-tangent-contact-force-KKT-reformulate}%
\end{subequations}
using the smooth norm $\snorm{\cdot}$ and smooth absolute value $\sabs{\cdot}$. The multiplication of $\phicn{i}$ in \eqref{eq:max-diss-tangent-contact-force-KKT-reformulate:stationarity} enforces stationarity only when contact is active, i.e., normal contact force is positive. Henceforth, we denote $\lambda^{\mr{c}}_i(q,v,\phic{i},\gamma_i) = \tangentcvel{i}(q,v) + \gamma_i\nabla\snorm{\phict{i}}^\top$ and $\rho^\mr{c}_i(\phic{i},\gamma_i) = \sabs{\mu_i\phicn{i}}-\snorm{\phict{i}}$.

Next, we eliminate the explicit joint reaction force in \eqref{eq:ct-dynamics:acc} by enforcing joint constraints at the acceleration level
\begin{subequations}
\begin{align}
    & \jointcnstr(q) = \zeros{\njoints} \implies \nabla\jointcnstr(q)v = \zeros{\njoints},  \label{eq:joint-cnstr-pos-vel}\\
    \implies & \nabla\jointcnstr(q)\dot{v} + (\ldots\,v^\top\nabla^2\jointcnstr_{[i]}(q)v\,\ldots) = \zeros{\njoints},\label{eq:joint-cnstr-acc}
\end{align}
\end{subequations}
where $\jointcnstr_{[i]}(q)$ is the $i$th component of $\jointcnstr$, for $i=1,\ldots,\njoints$. From \eqref{eq:joint-cnstr-acc} and \eqref{eq:ct-dynamics:acc} we solve for $\jointforce$ and substitute into \eqref{eq:ct-dynamics:acc}.

We rewrite \eqref{eq:ct-dynamics} in state-space form as $\dot{x}=f(x,u)$, with state $x=(q,v)\in\bR^{n^x}$, control input $u=(\jointtorque{},\phic{},\gamma)\in\bR^{n^u}$, and dimensions $n^x=2n^q$ and $n^u=\njointact+\ncontacts(\ndim+1)$. The state and control input are subject to path constraints $g(x(t),u(t))\le\zeros{n^g}$ for all $t\in\tspan$, which include, for each $i=1,\ldots,\ncontacts$, nonnegativity constraints on $\phicn{i}$, $\signdist(\posc{i}(q))$, $\gamma_i$, and $\rho^{\mr{c}}_i(\phic{i},\gamma_i)$, together with general constraints on state and input (e.g., obstacle avoidance, kinematic limits).
\subsection{Impulsive Impact Model}
We consider impulsive impacts occurring at discrete instants. Let an impact occur at time $\hat{t}$. Stack the contact impulses as $\Phic{}=(\Phic{1},\ldots,\Phic{\ncontacts})$ with $\Phic{i}=(\Phict{i},\Phicn{i})$, and let $\jointimpulse$ denote the joint reaction impulse. The impact dynamics are
\begin{subequations}
\begin{align}
 & M(q(\hat{t})) (v(\hat{t}^+) - v(\hat{t}^-)) = \Jc{}(q(\hat{t}))^\top\rotmatc{}(q(\hat{t}))\Phic{}\label{eq:impulse-dynamics:velocity-change}\\
 & \phantom{M(q(\hat{t})) (v(\hat{t}^+) - v(\hat{t}^-)) ={}}{}+\nabla \jointcnstr(q(\hat{t}))^\top\jointimpulse, \notag\\
 & 0 \le \Phicn{i} \perp \signdist(\posc{i}(q(\hat{t}))) \ge 0, \label{eq:impulse-dynamics:no-penetration}\\
 & 0 = \Phicn{i}\normalc{i}(q(\hat{t}))^\top \Jc{i}(q(\hat{t}))(v(\hat{t}^+)+\epsrest{i}v(\hat{t}^-)), \label{eq:elasticity-impact}\\
 & 0 = \Phicn{i}\left(\tangentcvel{i}(q(\hat{t}),v(\hat{t}^+)) + \Gamma_i\nabla\snorm{\Phict{i}}^\top\right), \label{eq:max-diss-impulse-KKT:stationarity}\\
 & 0 \le \Gamma_i \perp \sabs{\mu_i\Phicn{i}} - \snorm{\Phict{i}} \ge 0,\label{eq:impulse-friction-cone}\\
 & \qquad i = 1,\ldots,\ncontacts.\notag
\end{align}\label{eq:impulse-dynamics}%
\end{subequations}
Conditions \eqref{eq:max-diss-impulse-KKT:stationarity}–\eqref{eq:impulse-friction-cone} play the same role as \eqref{eq:max-diss-tangent-contact-force-KKT-reformulate}. They determine the tangential contact impulse via a stationarity condition (multiplied by normal contact impulse) and the friction cone constraint, with $\Gamma_i$ the associated Lagrange multiplier. Impact elasticity is governed by \eqref{eq:elasticity-impact}, where $\epsrest{i}$ is the coefficient of restitution at the $i$th contact.

We remove the explicit joint reaction impulse in \eqref{eq:impulse-dynamics:velocity-change} by imposing the joint constraint post impact at the velocity level,
$\nabla\jointcnstr(q(\hat{t}^+))\,v(\hat{t}^+)=\zeros{\njoints}$.
Solving the resulting constraint and \eqref{eq:impulse-dynamics:velocity-change} yields an explicit expression for $\jointimpulse$, which we substitute back into \eqref{eq:impulse-dynamics:velocity-change}.
%
\section{Integral Cross Complementarity}\label{sec:intg-cross-complementarity}
We seek a solution to \eqref{eq:ct-dynamics} and \eqref{eq:impulse-dynamics} over a time horizon $\tspan$. We discretize $\tspan$ into a grid of $N$ nodes, $0=t_1,\ldots,t_N=\tfinal$, and we isolate all contact mode switches to occur only at the grid nodes. These contact modes indicate either no contact, sticking contact or slipping contact.

First, we introduce the following continuous-time cross-complementarity condition. Let $a,b:\tspan\to\bR_+$ be piecewise continuous with any discontinuities confined to the grid. Specifically, for each $k=1,\ldots,N-1$, $a$ and $b$ are continuous on $(t_k,t_{k+1}]$. They satisfy pointwise-in-time complementarity
\begin{align}
    0 \le a(t) \perp b(t) \ge 0,\quad \forall\,t\in\tspan. \label{eq:a-b-complementarity}
\end{align}
To ensure switches between $a>0$, $b=0$ and $a=0$, $b>0$ happen at the grid nodes, we impose, for each $k=1,\ldots,N-1$,
\begin{align}
    0 \le a(t) \perp b(t^\prime) \ge 0,\quad \forall\,t,t^\prime\in (t_k,t_{k+1}], \label{eq:ct-cross-complementarity}
\end{align}
together with $0\le a(t_1)\perp b(t_1)\ge 0$. We then have the following continuous-time analogue of \cite[Lem.~2]{nurkanovic2024fesd-j}.
\begin{lemma}\label{lem:ct-cross-complementarity}
Let \eqref{eq:ct-cross-complementarity} hold and fix $k\in\{1,\ldots,N-1\}$. If $a(t)>0$ for some $t\in(t_k,t_{k+1}]$, then $b(t^\prime)=0$ for all $t^\prime\in(t_k,t_{k+1}]$. Conversely, if $b(t^\prime)>0$ for some $t^\prime\in(t_k,t_{k+1}]$, then $a(t)=0$ for all $t\in(t_k,t_{k+1}]$.
\end{lemma}
\begin{proof}
If $a(t)>0$ at some $t\in(t_k,t_{k+1}]$ and $b(t^\prime)>0$ at some $t^\prime\in(t_k,t_{k+1}]$, then \eqref{eq:ct-cross-complementarity} is violated, which is a contradiction.
\end{proof}
While \eqref{eq:ct-cross-complementarity} localizes switches to the grid nodes (a desirable property for both numerical reliability and physical realism), it is still a pointwise-in-time constraint and thus numerically intractable. Existing approaches \cite{baumrucker2009mpec,nurkanovic2024fesd-j,kazi2025optimal} enforce discrete-time surrogates of \eqref{eq:ct-cross-complementarity} in aggregated forms, which risks inter-sample violations. To avoid this, we propose an integral reformulation, referred to as an \emph{integral cross-complementarity} condition.
\begin{proposition}
For each $k=1,\ldots,N-1$, the constraint
\begin{align}
    0 \le \left(\int_{t_k}^{t_{k+1}} a(t)\,\mr{d}t\right) \perp \left(\int_{t_k}^{t_{k+1}} b(t)\,\mr{d}t\right) \ge 0,
    \label{eq:intg-cross-complementarity}
\end{align}
is equivalent to \eqref{eq:ct-cross-complementarity}.
\end{proposition}
\begin{proof}
($\Rightarrow$) If \eqref{eq:intg-cross-complementarity} holds and $\int_{t_k}^{t_{k+1}} a(t)\,\mr{d}t>0$, then the complementarity of the integrals implies $\int_{t_k}^{t_{k+1}} b(t)\,\mr{d}t=0$. Since $a$ and $b$ are nonnegative and are continuous on $(t_k,t_{k+1}]$, it follows that $a(t)>0$ and $b(t) = 0$ for all $t \in (t_k,t_{k+1}]$. The argument is symmetric if $\int_{t_k}^{t_{k+1}} b(t)\mr{d}t>0$. Therefore, \eqref{eq:ct-cross-complementarity} holds.

($\Leftarrow$) If \eqref{eq:ct-cross-complementarity} holds, Lemma~\ref{lem:ct-cross-complementarity} implies that for each $k=1,\ldots,N-1$, either $a(t) = 0$ or $b(t) = 0$ (or both) for all $t\in(t_k,t_{k+1}]$. Hence one (or both) of the integrals is zero, yielding \eqref{eq:intg-cross-complementarity}.
\end{proof}
In practice, we compute the integrals in \eqref{eq:intg-cross-complementarity} by integrating an auxiliary dynamical system
\begin{align}
    \dot{y} = (\dot{y}_{[1]},\dot{y}_{[2]})= (a,b),\label{eq:aux-dynamics-eg}
\end{align}
and then impose the complementarity constraint
\begin{align}
0 \le y_{[1]}(t_{k+1}) - y_{[1]}(t_k)\perp y_{[2]}(t_{k+1}) - y_{[2]}(t_k) \ge 0.\label{eq:aux-system-complementarity-eg}%
\end{align}
If $a$ or $b$ are not inherently nonnegative, we enforce their nonnegativity through the path constraint function $g$. In \eqref{eq:ct-dynamics}, to ensure that contact mode switches occur only at the grid nodes, we impose integral cross-complementarity between the following pairs: (i) $\phicn{i}$ and $\signdist(\posc{i}(q))$, (ii) $\phicn{i}$ and $\lambda_i^{\mr{c}}(q,v,\phic{i},\gamma_i)$, and (iii) $\gamma_i$ and $\rho_i^{\mr{c}}(\phic{i},\gamma_i)$.
\begin{remark}
Discrete-time cross-complementarity formulations (e.g., FESD-J \cite{nurkanovic2024fesd-j}) typically introduce additional cross-complementarity constraints to detect sign changes in the tangential velocity at each contact point. Our formulation obviates these by parameterizing the control input $u$ on each interval $(t_k,t_{k+1}]$, $k=1,\ldots,N-1$, with $\mc{C}^1$ basis functions. Consequently, $\lambda_i^{\mr{c}}$ is also $\mc{C}^1$ on $(t_k,t_{k+1}]$, which precludes nonsmoothness of the tangential velocity between nodes.
\end{remark}
%
\section{Contact-Implicit Trajectory Optimization Problem}
In this section, we formulate the CITO problem for the maximal-coordinate model of Section~\ref{sec:system-model}. The key steps are: (i) constructing an auxiliary system to handle path constraints and integral cross-complementarity, (ii) forming the time-dilated augmented system, (iii) parameterizing the control input and the dilation factor, and (iv) applying a smooth relaxation of the complementarity constraints. We employ time-dilation \cite{kamath2023real-time,elango2025ctscvx} so that the time-discretization associated with integral cross complementarity does not constrain the contact mode sequence (e.g., gait cadence and pattern).
%
\subsection{Auxiliary System}
We first construct the auxiliary system, which serves two purposes: (i) encoding cumulative violation of path constraints \cite[Lem. 2]{elango2025ctscvx}, and (ii) representing the integral cross-complementarity constraints (see \eqref{eq:aux-dynamics-eg}). The auxiliary dynamics are%
\begin{align}
    \dot{y} &=
    \begin{bmatrix}
        \vdots\\[-0.1cm]
        \dot{y}_{[4(i-1)+j]}\\[-0.1cm]
        \vdots\\[-0.25cm]\hdotsfor{1}\\[-0.225cm]
        \vdots\\[-0.1cm]
        \dot{y}_{[4\ncontacts+k]}\\[-0.1cm]
        \vdots
    \end{bmatrix}
    =
    \begin{bmatrix}
        \vdots\\
        \phicn{i}\\
        \lambda^{\mr{c}}_i(q,v,\phicn{i},\gamma_i)\\
        \gamma_i\\
        \rho^{\mr{c}}_i(\phic{i},\gamma_i)\\[-0.1cm]\vdots\\[-0.25cm]\hdotsfor{1}\\[-0.05cm]
        M^g\,\sabsplus{g(x,u)}
    \end{bmatrix}
    = \Omega(x,u),
    \label{eq:aux-dynamics}
\end{align}
where $i=1,\ldots,\ncontacts$, $j=1,\ldots,4$, $k=1,\ldots,N^g$, and $M^g\in\bR^{N^g\times n^g}$ is a nonnegative ``mixing'' matrix that allocates the $n^g$ components of $\sabsplus{g(x,u)}$ across $N^g$ components of the auxiliary state. Let $n^y=4\ncontacts+N^g$ denote the dimension of the auxiliary state. Continuous-time satisfaction of the path constraints on each grid interval $(t_k,t_{k+1}]$ is enforced by
$y_{[4\ncontacts+i]}(t_{k+1})=y_{[4\ncontacts+i]}(t_k)$ for $i=1,\ldots,N^g$ and $k=1,\ldots,N-1$. To preserve constraint qualification, we relax this to
\begin{align}
    y_{[4\ncontacts+i]}(t_{k+1})-y_{[4\ncontacts+i]}(t_k)\le \epsilon,\label{eq:ctcs-relax}
\end{align}
with a positive tolerance $\epsilon$. With appropriate scaling and preconditioning of the resulting CITO problem, $\epsilon$ can be chosen as a small but numerically meaningful value (e.g., $\sim 10^{-4}$), which leads to negligible continuous-time constraint violation. See \cite{elango2025ctscvx} for further discussion and analysis.
%
\subsection{Parameterized Time-Dilated Augmented System}
We define the time-dilated augmented system on the normalized interval $[0,1]$ with augmented state $\tilde{x}=(x,y)$ and augmented control $\tilde{u}=(u,s)$
\begin{align}
    \derv{\tilde{x}} = s\,\big(f(x,u),\,\Omega(x,u)\big) = \tilde{f}(\tilde{x},\tilde{u}),
    \label{eq:aug-dynamics}
\end{align}
where $\derv{\square}=\mr{d}\square/\mr{d}\tau$ for $\tau\in[0,1]$, and $s=\mr{d}t/\mr{d}\tau$ is the dilation factor, treated as a control input. Let $n^{\tilde{x}}=n^x+n^y$ and $n^{\tilde{u}}=n^u+1$.

We form an integral representation of \eqref{eq:aug-dynamics} over each grid interval from Section \ref{sec:intg-cross-complementarity}. First, for each $k=1,\ldots,N-1$, we map $[0,1]$ to $[t_k,t_{k+1}]$, and parameterize the augmented control on $[0,1]$ as
\[
\tilde{u}(\tau) = \big(\nu(\tau)U,\,S\big),
\]
where $\nu:[0,1]\to\bR^{\,n^u\times n^uN^u}$ is a matrix-valued polynomial basis function, and $U\in\bR^{n^uN^u}$ and $S\in\bR$ are the corresponding coefficients. Then the integral form of \eqref{eq:aug-dynamics} over $[0,1]$ is
\begin{align}
    \tilde{F}(\tilde{x}^-,U,S)
    = \tilde{x}^- + \int_{0}^{1} \tilde{f}\big(\tilde{x}(\tau),\,(\nu(\tau)U,S)\big)\,\mr{d}\tau,\label{eq:intg-aug-sys-param}
\end{align}
for $\tilde{x}^-\in\bR^{n^{\tilde{x}}}$, $U\in\bR^{n^uN^u}$, and $S\in\bR$.
%
\subsection{Nonconvex Optimization Problem}
We can now cast the CITO problem as a nonconvex optimization problem. A key aspect of most existing approaches for CITO is a smooth relaxation of the complementarity constraints. Complementarity is inherently nonsmooth: for scalars $a$ and $b$ the condition $0\le a\perp b\ge 0$ is equivalent to $\min(a,b)=0$, i.e., feasible set is the nonnegative coordinate axes. We employ the well-known $\mathcal{C}^1$ Fischer–Burmeister function to relax a complementarity constraint
\begin{align}
    \mr{FB}^\delta(a,b) \le 0,~a\ge 0,~b\ge 0,\label{eq:FB-relax}
\end{align}
where $\mr{FB}^\delta(a,b)=a+b-\sqrt{a^2+b^2+\delta}$ and $\delta>0$ controls the tightness of the relaxation. As $\delta\to 0$ the feasible set of \eqref{eq:FB-relax} approaches that of $\min(a,b)=0$. Note, any other $\mc{C}^1$ nonlinear complementarity function could be used in place of Fischer–Burmeister.

The decision variables of the CITO problem are:
(i) pre- and post-impact augmented states at the grid nodes
$\tilde{x}_1^-,\tilde{x}_1^+,\ldots,\tilde{x}_N^-,\tilde{x}_N^+$,
(ii) coefficients parameterizing the augmented control input on each grid interval
$U_1,S_1,\ldots,U_{N-1},S_{N-1}$, and
(iii) contact impulses and associated multipliers at the nodes
$\Phic{1},\Gamma_1,\ldots,\Phic{N},\Gamma_N$.
Here $\tilde{x}_k^\square=(x_k^\square,y_k^\square)$ with $x_k^\square=(q_k^\square,v_k^\square)$,
$\Phic{k}=(\Phic{k,1},\ldots,\Phic{k,\ncontacts})$ with 
$\Phic{k,i}=(\Phict{k,i},\Phicn{k,i})$, and
$\Gamma_k=(\Gamma_{k,1},\ldots,\Gamma_{k,\ncontacts})$,
for $\square\in\{{\scriptstyle+,-}\}$, $k=1,\ldots,N$, and $i=1,\ldots,\ncontacts$.

The CITO problem comprises of the following constraints. First, the augmented dynamics is enforced in integral form \eqref{eq:intg-aug-sys-param}, for $k=1,\ldots,N-1$. Next,  \eqref{eq:impulse-dynamics:velocity-change} is imposed at each grid node to allow potential impacts, along with continuity of pose and auxiliary state across impact, via
$\tilde{G}(\tilde{x}_k^+,\tilde{x}_k^-,\Phic{k})=\zeros{n^{\tilde{x}}}$, for $k=1,\ldots,N$. Next, we require the parameterization coefficients to lie in convex sets, i.e.,
$U_k\in\mc{U}$ and $S_k\in\mc{S}$, for $k=1,\ldots,N-1$. Set $\mc{U}$ encodes joint torque bounds, nonnegativity of the multiplier $\gamma$, and the contact force friction cone. Set $\mc{S}$ is a positive compact interval. Next, integral cross complementarity, relaxed via \eqref{eq:FB-relax}, and the relaxed path constraint satisfaction \eqref{eq:ctcs-relax} are aggregated as $\Psi^{\delta,\epsilon}(y_{k+1}^-,y_k^+)\le \zeros{n^\Psi}$, for $k=1,\ldots,N-1$. Next, the conditions \eqref{eq:impulse-dynamics:no-penetration}-\eqref{eq:impulse-friction-cone} are collected into $\Theta^\delta(x_k^+,x_k^-,\Phic{k},\Gamma_k)\le \zeros{n^\Theta}$ using relaxation \eqref{eq:FB-relax}. Note that \eqref{eq:elasticity-impact} and \eqref{eq:max-diss-impulse-KKT:stationarity} are also expressed as complementarity constraints. Finally, we constrain the initial and terminal states $x_{\mr{i}}$ and $x_{\mr{f}}$.

The CITO problem is then given by
\begin{subequations}
\begin{align}
\underset{\genfrac{}{}{0pt}{}{\scriptstyle\tilde{x}^+_k,~\tilde{x}^-_k,~U_k}{\scriptstyle S_k,~\Phic{k},~\Gamma_k}}{\mr{minimize}}~&~\sum_{k=1}^{N-1}L_k(x_k^+,U_k,S_k)\label{eq:cito-problem-obj}\\
\mr{subject~to}~&~\tilde{x}{}^-_{k+1}=\tilde{F}(\tilde{x}^+_k,U_k,S_k),\label{eq:cito-augdyn}\\
 &~\Psi^{\delta,\epsilon}(y_{k+1},y_k) \le \zeros{n^\Psi}, \\
 &~U_{k}\in\mc{U},~S_k\in\mc{S}, \\
 &~\qquad k= 1,\ldots,N-1,\notag\\
 &~\tilde{G}(\tilde{x}^+_{k},\tilde{x}^{-}_{k},\Phic{k}) = \zeros{n^{\tilde{x}}},\label{eq:cito-impulse}\\
 &~\Theta^\delta(x^+_k,x^-_k,\Phic{k},\Gamma_k) \le \zeros{n^\Theta},\\
 &~\qquad k=1,\ldots,N,\notag\\ 
 &~x^-_1=x_{\mr{i}},~x^+_N=x_{\mr{f}}.
\end{align}\label{eq:cito-problem}%
\end{subequations}
The cost function \eqref{eq:cito-problem-obj} can represent various objectives, such as control effort, reference tracking, or maneuver time.
\begin{figure*}[t!] 
    \centering 
    \begin{subfigure}[t]{\textwidth} 
        \centering 
        \includegraphics[width=.8\linewidth]{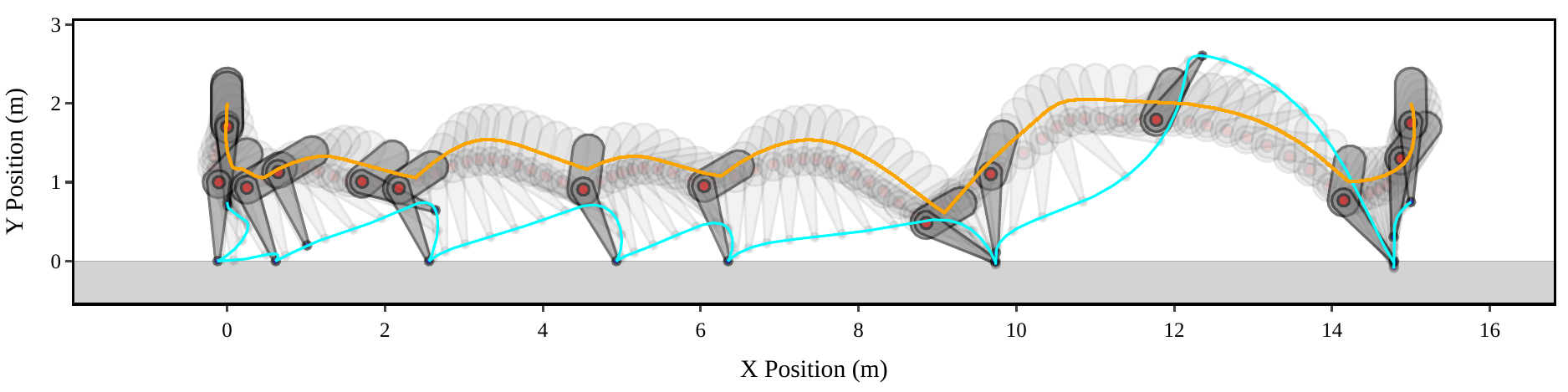}
        \caption{Gait produced by a monoped locomotion solution. The path of the center of mass (CoM) of the upper leg is shown by the \textcolor{orange}{orange} trace, and the path of the contact point (end of the lower leg) is shown by the \textcolor{cyan}{cyan} trace. The pose at grid nodes are displayed with higher opacity.} 
        \label{fig:monoped-gait} 
    \end{subfigure} 
    \begin{subfigure}[t]{.48\textwidth} 
        \centering 
        \includegraphics[width=\linewidth]{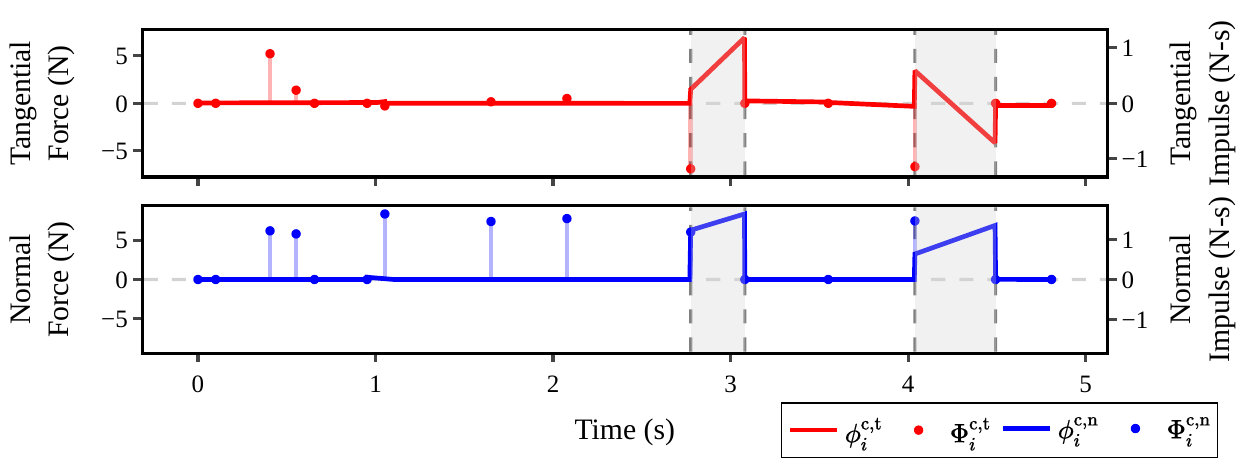} 
        \caption{Time histories of the normal and tangential contact forces and impulses. Grid intervals with sustained contact are denoted by gray regions.} \label{fig:monoped-contact} 
    \end{subfigure} 
    \hspace{.01\textwidth} 
    \begin{subfigure}[t]{.48\textwidth} 
        \centering 
        \includegraphics[width=\linewidth]{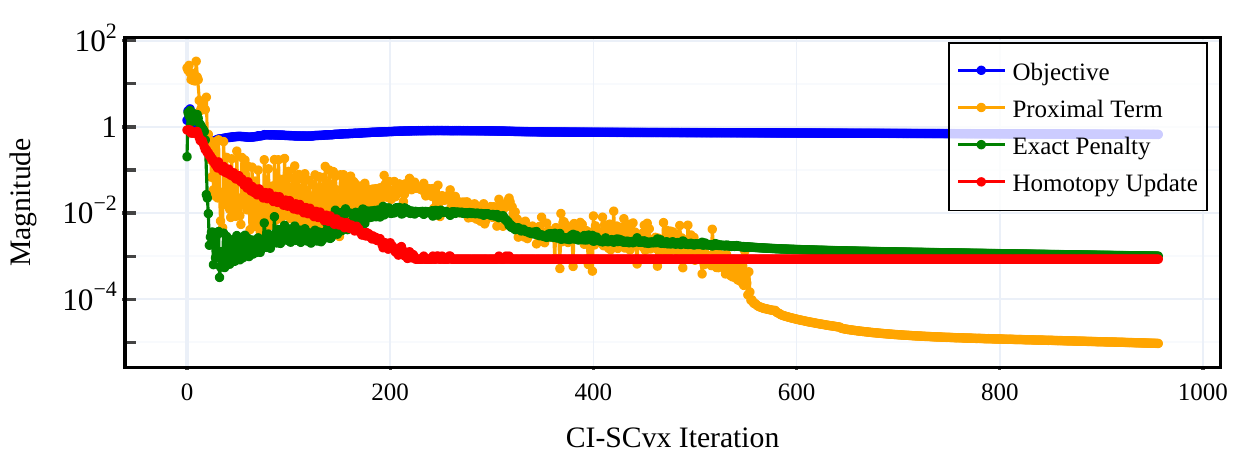} 
        \caption{History of convergence metrics across {\ciscvx} iterations, including the subproblem objective terms and homotopy tightness parameter $\delta_j$.} \label{fig:monoped-scvx} 
    \end{subfigure} 
    \caption{Monoped locomotion solution on a flat surface, produced by {\ciscvx}.}
    \label{fig:monoped-all}
    \vspace{-6mm}
\end{figure*}
\section{Sequential Convex Programming with Backtracking Homotopy}
To solve the nonconvex problem \eqref{eq:cito-problem}, we employ sequential convex programming (SCP), specifically the prox-linear method with exact penalization of nonconvex constraints \cite[Sec.~4]{elango2025ctscvx}. Each SCP iteration solves a convex subproblem in which the nonconvex constraints are linearized and penalized with an $\ell_1$ (exact) penalty, the convex constraints from \eqref{eq:cito-problem} are kept intact, and a quadratic proximal term penalizes deviation from the previous iterate. With suitable weights for the proximal and exact-penalty terms, the iterates converge to a stationary point of the penalized nonconvex problem. If that point is feasible for \eqref{eq:cito-problem}, it is a KKT point.

Applying SCP to problems with complementarity constraints requires special care. When using a relaxation such as \eqref{eq:FB-relax}, fixing the tightness parameter $\delta$ a priori often stalls the algorithm and prevents convergence. For reliable convergence, $\delta$ is typically driven toward zero via a homotopy over the SCP iterations \cite{malyuta2023fast,hall2021sequential}. We propose a homotopy scheme for $\delta$, embedded within SCP (referred to as embedded $\delta$-homotopy) with backtracking, in which $\delta$ is tightened progressively but can be temporarily relaxed to allow larger steps and broader search directions before committing to a smaller value. Algorithm~\ref{alg:homotopy} describes the approach. Let $z_j$, $\delta_j$, $J^{\mr{px}}_j$, and $J^{\mr{ep}}_j$ denote the subproblem solution, the tightness parameter, the proximal term, and the exact-penalty term, respectively, at SCP iterate $j$. Inputs to the algorithm are the initial tightness parameter $\delta_0$, an initial guess $z_0$, tolerances $\delta^{\min}(<\!\delta_0)$, $\epsilon^{\mr{px}}$, $\epsilon^{\mr{ep}}$, a shrinkage factor $\alpha\in(0,1)$, and a stall horizon $N^{\mr{stall}}$. The update rule for the tightness parameter  has three parts: (i) if neither $J^{\mr{px}}_j$ nor $J^{\mr{ep}}_j$ has decreased relative to iterate $M = \max(j-N^{\mr{stall}},1)$, set $\delta_{j+1}$ to the log-scale bisection between $\delta_j$ and $\delta_{M}$, (ii) else, if neither term has decreased relative to the previous iterate, set $\delta_{j+1}=\delta_j$; (iii) otherwise, shrink the tightness parameter as $\delta_{j+1}=\alpha\delta_j$.
%
\begin{algorithm}[H]
    \captionsetup{font=small}
    \caption{$\delta$-Homotopy}
    \footnotesize 
    \begin{algorithmic}[1]  
    \REQUIRE $z_0$, $\delta_0$, $\delta^{\min}$, $\epsilon^{\mr{px}}$, $\epsilon^{\mr{ep}}$, $\alpha$, $N^\mathrm{stall}$
    \ENSURE $z_{j}$
    \STATE $j\gets0,~J^{\mr{px}}_j\gets0,~J^{\mr{ep}}_j\gets0$
    \WHILE{$J_j^\mr{px}>\epsilon^{\mr{px}}\:\vee\:J_j^\mr{ep}>\epsilon^{\mr{ep}}\:\vee\:\delta_j > \delta^{\min}$}
            \STATE $z_{j+1},J^{\mr{px}}_{j+1},J^{\mr{ep}}_{j+1}\gets\texttt{solve\_subproblem}(z_j,\delta_j)$
            \STATE $M \gets \max(j-N^{\mathrm{stall}},1)$
            \IF{$J^\mathrm{px}_{j+1} > J^\mathrm{px}_{M}\:\vee\:J^\mathrm{ep}_{j+1} > J^\mathrm{ep}_{M}$}
                \STATE$\delta_{j+1} \gets \text{exp}\left(\frac{1}{2}\log(\delta_{j}\delta_{M})\right)$
            \ELSIF{$J^\mathrm{px}_{j+1} > J^\mathrm{px}_{j}\:\vee\:J^\mathrm{ep}_{j+1} > J^\mathrm{ep}_{j}$}
                \STATE $\delta_{j+1} \gets \delta_{j}$
            \ELSE
                \STATE $\delta_{j+1} \gets \alpha \delta_{j}$
            \ENDIF
            \STATE $j \gets j+1$\\
        \ENDWHILE
    \end{algorithmic}
    \label{alg:homotopy}
\end{algorithm}

\section{Experimental Results}\label{sec:experimental}
In this section, we detail the implementation of the proposed {\ciscvx} framework and present case studies on trajectory optimization for monopedal and bipedal multi-body robots, illustrating the versatility of {\ciscvx}. A notable advantage of our algorithm and formulation is its suitability for GPU-parallelized execution. The key computational task in solving \eqref{eq:cito-problem} is transcription, i.e., constructing and linearizing \eqref{eq:cito-augdyn} and \eqref{eq:cito-impulse}, which define the hybrid dynamical model due to contact and impact. We carry out transcription on the GPU using a JAX backend by exploiting two features of the formulation: (i) multiple shooting enables parallelization across grid intervals, and (ii) the maximal-coordinate representation enables parallel evaluation of model elements (joint constraints, contact Jacobians, etc.). Transcription on each grid interval uses Internal Numerical Differentiation (IND) \cite{quirynen2017numerical} with a custom Runge-Kutta-Fehlberg integrator and a first-order-hold (FOH) control input parameterization that is allowed to be discontinuous at grid nodes. Although more advanced JAX-compatible integrators (e.g., Diffrax) are available, in our experiments they were slower and did not improve SCP convergence. Beyond transcription, the remaining nonconvex constraints in \eqref{eq:cito-problem} are also linearized on the GPU.
\begin{figure*}[t!]
    \centering
    \begin{subfigure}[t]{.44\textwidth}
        \centering
        \includegraphics[width=\textwidth]{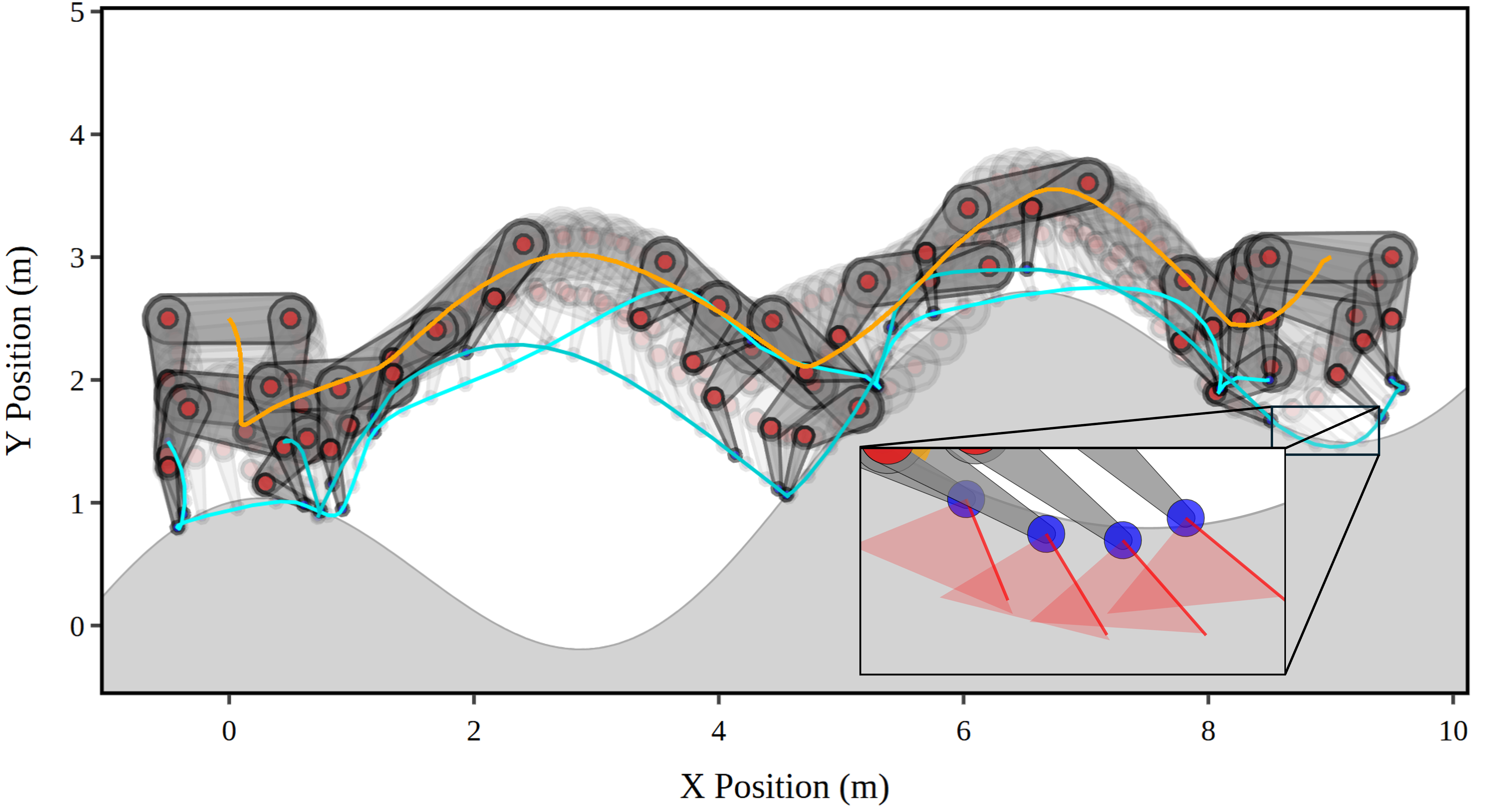}
        \caption{Gait produced by a biped locomotion solution. The path of the torso CoM is shown by the \textcolor{orange}{orange} trace, and the paths of the contact points (front and hind feet) are shown by the \textcolor{cyan}{cyan} traces. The zoomed view shows slipping with contact forces activating the friction cone.}
        \label{fig:biped-gait}
    \end{subfigure}
    \hspace{.05\textwidth}
    \begin{subfigure}[t]{.44\textwidth}
        \centering
        \includegraphics[width=\linewidth]{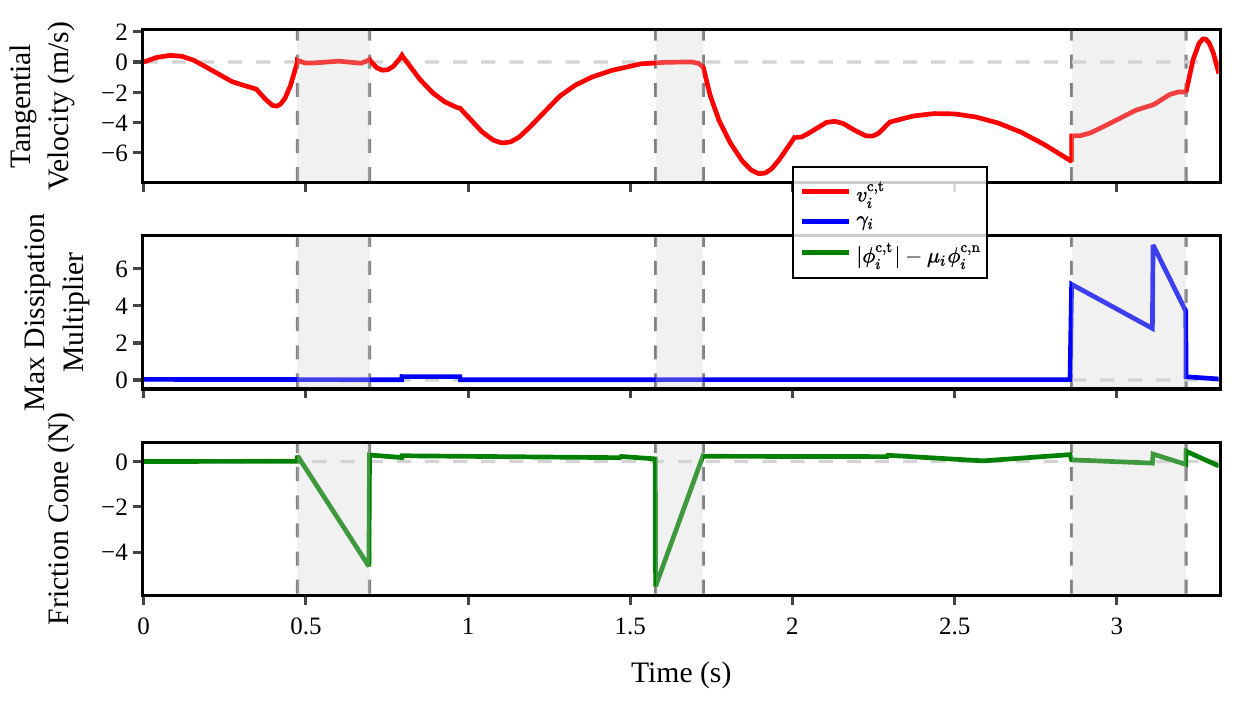}
        \caption{Time histories of tangential velocity at contact point, max dissipation Lagrange multiplier and the friction cone term for the front foot, showing transitions between stick and slip motion. Impulses are not displayed as they are approximately zero for this scenario.}
        \label{fig:biped-contacts}
    \end{subfigure}
    \caption{Biped locomotion solution on a sinusoidal surface, produced by {\ciscvx}.}\label{fig:biped} 
    \vspace{-4mm}
\end{figure*}
All linearized quantities are transferred from the GPU to the CPU and assembled, together with the convex constraints of \eqref{eq:cito-problem}, into a canonical-form second-order cone program (SOCP) using compressed sparse column (CSC) format for the constraint and objective matrices. This process is referred to as canonicalization. During assembly, we apply row-wise normalization (preconditioning) to all constraint matrices and affine scaling to the decision variables, which improves solver reliability and the practical convergence of SCP. The resulting sparse data are passed directly to the QOCO convex solver \cite{chari2025qoco}, allowing us to bypass general-purpose modeling tools (e.g., CVXPY) and produce high-fidelity solutions in near real time (tens of seconds). Figure \ref{fig:architecture} illustrates the full \ciscvx\:architecture.

All results are initialized by linearly interpolating poses at boundary conditions, setting velocities and joint torques to zero, and randomly sampling contact forces/impulses from the friction cone, which demonstrates that realistic trajectories can be obtained from naive initializations. A grid size of $N=15$ and integration step count of $10$ (per each grid interval) is chosen for all cases, demonstrating that we can resolve complex motions under sparse discretization. Each scenario minimizes the cumulative squared joint torques, with proximal and exact-penalty weights in the SCP subproblems set to $2000$ and $50,000$, respectively. We also globally set $\delta_0=1$, $\delta^{\min}=10^{-3}$, $N^{\mr{stall}}=10$, $\alpha=0.85$, $\mu_i=1$ and $\epsilon_i^{\mathrm{r}}=0$, for $ i=1,\dots,n^\mathrm{c}$. All results are visualized with the Plotly library in Python and forward-simulated with time steps of $0.01$ s.
%
\subsection{Monoped on a Flat Surface}
We consider locomotion for a simple two-link monoped (Fig.~\ref{fig:monoped-all}) on a flat contact surface. Fig.~\ref{fig:monoped-gait} shows a largely periodic gait in which contact is predominantly instantaneous (impulsive) over the horizon. The corresponding impulses (Fig. \ref{fig:monoped-contact}) are nonzero over most of the trajectory, while contact forces remain zero, yielding an energy-efficient ``touch-and-go'' motion driven by torque minimization. Although impacts are inelastic, an apparent ``bouncing'' motion is produced due to the use of joint torques.  Sustained contact (gray window) occurs only in two intervals at the lift-off and landing points for the final lateral jump, at which point both the contact forces and impulses are activated. During this jump, the lower leg of the monoped lifts close to the upper leg to shift the system mass distribution.

Convergence metrics for {\ciscvx} are shown in Fig. \ref{fig:monoped-scvx}. The tightness parameter $\delta$ decreases steadily with minor oscillations due to the backtracking rule in Algorithm~\ref{alg:homotopy}. Solutions are obtained in roughly 40\,s on average, over about 800 {\ciscvx} iterations. We apply the termination criteria on scaled variables, which results in high-accuracy solutions.  Lower-accuracy yet physically reasonable trajectories for closed-loop tracking can be generated more quickly.
\subsection{Biped on a Sinusoidal Surface}
Our second case study considers a five-link biped traversing a sinusoidal surface (Fig.~\ref{fig:biped}). We approximate the signed distance of contact point to the contact surface as
\begin{equation}
    \signdist(\posc{i}(q)) \;\approx\; \frac{h^{\mathrm{c}}(\posc{i}(q))}{\|\nabla h^{\mathrm{c}}(\posc{i}(q))\|}.
\end{equation}
This choice (i) becomes arbitrarily accurate as $h^{\mathrm{c}}(\posc{i}(q))\!\to\!0$, (ii) is numerically stable far from the contact surface, and (iii) avoids solving a nonlinear optimization problem. This scenario showcases the capability of {\ciscvx} to produce solutions with stick–slip transitions. The resulting gait (Fig.~\ref{fig:biped-gait}) combines sticking and slipping along the uneven terrain to achieve feasible, efficient motion. In Fig.~\ref{fig:biped-contacts}, the tangential velocity is (approximately) zero during the first two persistent-contact phases (gray regions), indicating stick. In the third phase, slipping is activated over two grid intervals: the friction cone term approaches its boundary (i.e., $\abs{\mu_i\phicn{i}}-\norm{\phict{i}}\!\approx\!0$), the maximum dissipation multiplier $\gamma_i$ becomes nonzero, and the tangential velocity decreases toward zero. This behavior demonstrates that {\ciscvx} can automatically determine the appropriate contact mode.
\begin{figure*}[t!] 
    \centering 
    \label{fig:monoped} 
    \caption*{} 
    \begin{subfigure}[t]{\textwidth} 
        \centering 
        \includegraphics[width=.85\linewidth]{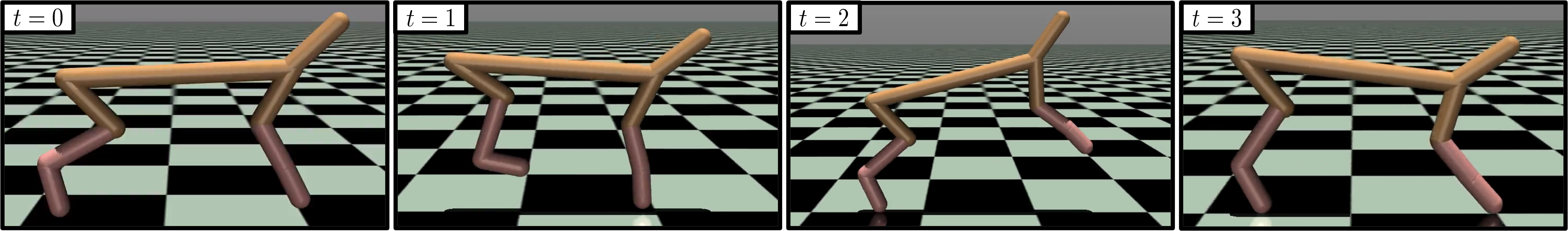} 
        \caption{Snapshots of the MJPC-tracked {\ciscvx} solution visualized with Gymnasium.} 
        \label{fig:cheetah-gait-mujoco} 
    \end{subfigure} 
    \begin{subfigure}[t]{.48\textwidth} 
        \centering 
        \includegraphics[width=\linewidth]{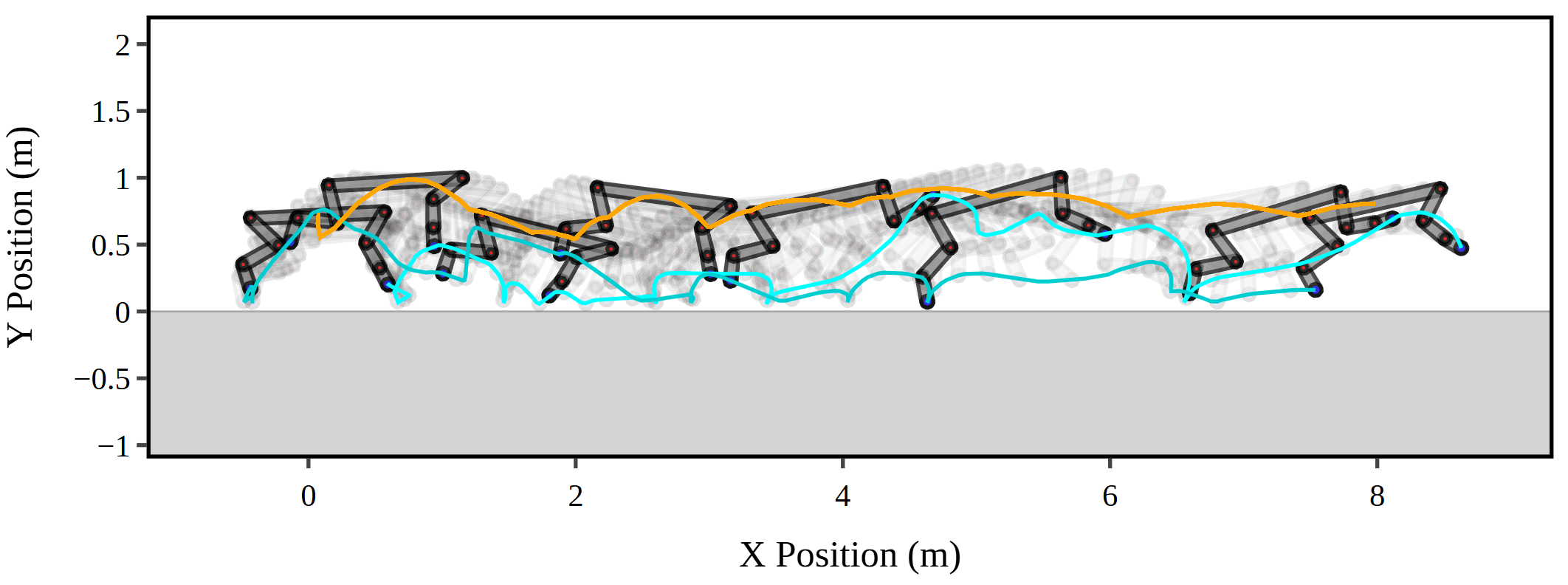} 
        \caption{Gait produced by a standalone MJPC forward rollout.} 
        \label{fig:cheetah-mjpc} 
    \end{subfigure} 
    \hspace{.01\textwidth}
    \begin{subfigure}[t]{.48\textwidth} 
        \centering 
        \includegraphics[width=\linewidth]{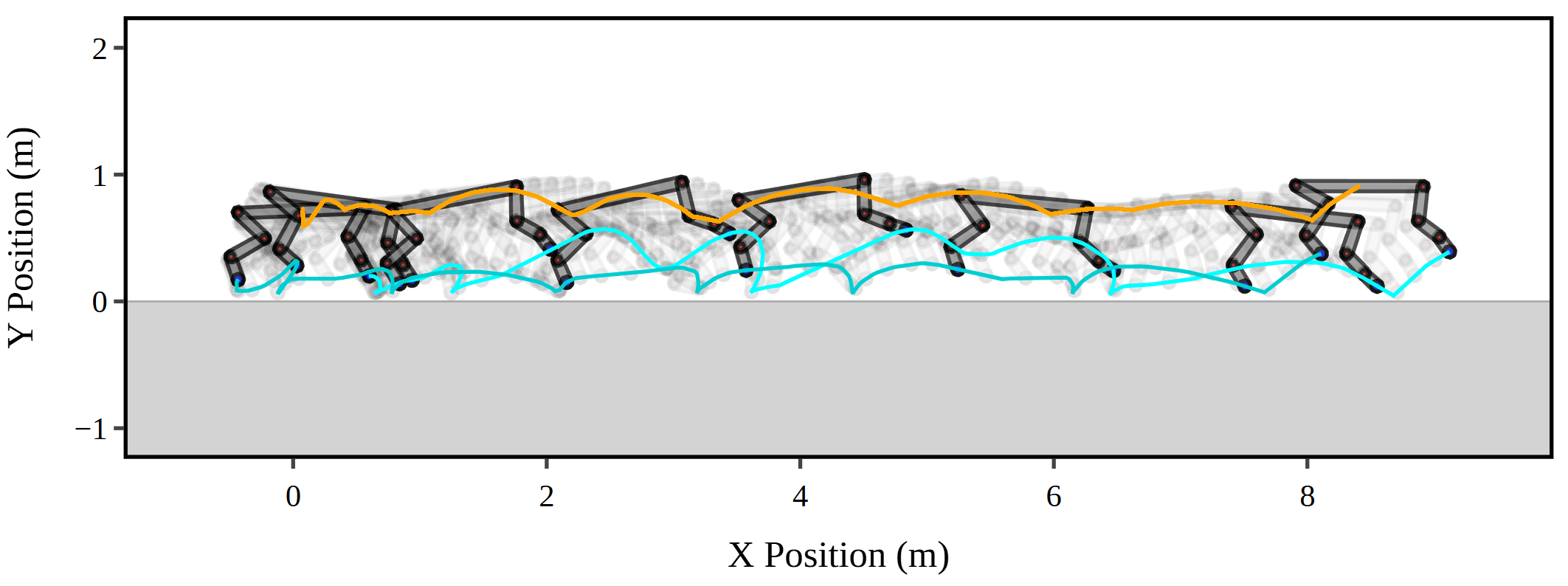} 
        \caption{Gait produced by MJPC tracking of a {\ciscvx} solution.} 
        \label{fig:cheetah-scvx} 
    \end{subfigure} 
    \caption{Comparison between solutions produced for the Gymnasium HalfCheetah environment. We note that the ``head'' of the model is not visualized in (b)-(c) as it is massless and does not contribute to the model dynamics.\label{fig:cheetah}}
    \vspace{-6mm}
\end{figure*}
\subsection{Demonstration with MuJoCo}
The previous case studies demonstrated the expressivity of our formulation in simplified settings. Here we assess the applicability of {\ciscvx} in a high-fidelity environment. Specifically, we use the Gymnasium HalfCheetah environment \cite{towers2024gymnasium} with a MuJoCo backend, together with Google DeepMind’s MuJoCo MPC (MJPC) toolbox and its predictive sampling algorithm \cite{howell2022predictive}. We employ MJPC for two purposes: (i) closed-loop tracking (with an LQR-style objective) of a given {\ciscvx} trajectory, and (ii) a baseline controller optimized with a standalone reward which balances the default environment reward (maximize forward velocity, minimize normalized joint torques) with a gait-phasing term that encourages alternating foot contacts. Sample rollouts are discarded if the torso exceeds a critical pitch angle or drops below a critical height (indicating tip-over and slumping failure modes).

Fig.~\ref{fig:cheetah} compares an MJPC forward rollout using the standalone reward to an MJPC rollout that tracks a {\ciscvx} solution, both on a horizon of $t^{\mathrm{f}}=3.2\,\mathrm{s}$. The {\ciscvx}-tracked rollout requires less reward shaping than the standalone MJPC baseline and still achieves a 6.64\% improvement in terms of the default environment objective as well as a 10.52\% improvement in the cumulative joint torque cost. Qualitatively, the {\ciscvx} trajectory yields a more stable gait, as illustrated in Fig.~\ref{fig:cheetah-scvx}.
\subsection{Timing Comparisons to Existing SCP Implementations}
Finally, we compare the fully GPU-accelerated, customized version of {\ciscvx} described in Section \ref{sec:experimental}, which we refer to as {\ciscvxcustom}, against a baseline version (\ciscvxbase), which inherits from implementation approaches in OpenSCvx v0.2.0 \cite{hayner2025continuous}, a state-of-the-art SCP toolbox for trajectory optimization. Specifically, we compare against OpenSCvx’s GPU parallelization and convex subproblem parsing pipeline (which relies on CVXPY). As shown in Fig.~\ref{fig:solvetimecompare}, {\ciscvxcustom} achieves speedups of approximately $6\times$ (monoped) and $19\times$ (biped). Notably, CVXPY fails to parse the subproblem for models more complex than the biped. The primary advantages of {\ciscvxcustom} relative to {\ciscvxbase} arise from (i) an efficient JAX-based implementation for transcription and nonconvex constraint linearization, and (ii) direct canonicalization of subproblem that bypasses general-purpose parsers such as CVXPY.

\begin{figure}[!htpb]
    \centering
    \includegraphics[width=0.75\linewidth]{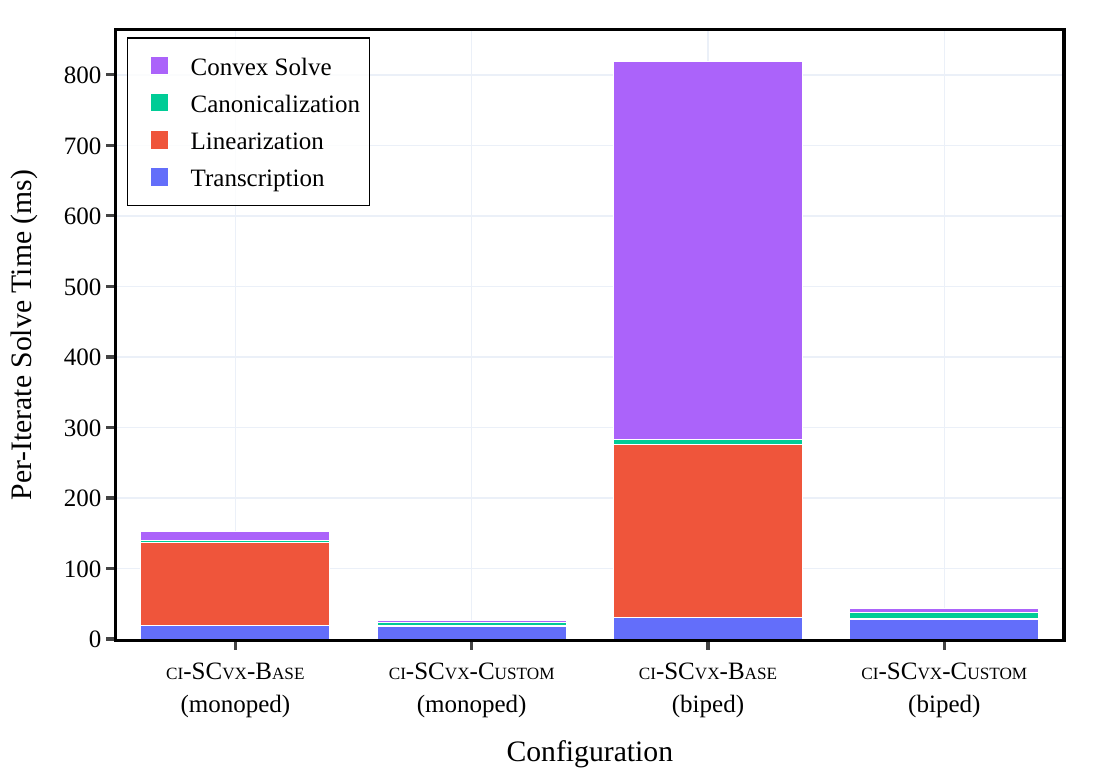}
    \caption{Timing results comparing the proposed {\ciscvxcustom} implementation against the baseline {\ciscvxbase}.}
    \label{fig:solvetimecompare}
    \vspace{-4mm}
\end{figure}
\section{Conclusion}
In this work, we developed a successive convexification–based framework for contact-implicit trajectory optimization (CITO), {\ciscvx}, which introduces integral cross-complementarity constraints to obtain high-fidelity solutions with sparse time discretization. The approach also employs a backtracking homotopy scheme to ensure reliable convergence. We demonstrated {\ciscvx} on complex locomotion tasks, showing performance and solve-time improvements over baselines in the literature. Future work will refine the homotopy strategy and solver implementation to reduce runtime and expand benchmarking to a broader set of models and trajectory optimization solvers.
\printbibliography

\end{document}